\documentclass[submission,copyright,creativecommons]{eptcs}
\providecommand{\event}{TARK 2019} %
\usepackage{breakurl}             %

\usepackage{xspace}
\usepackage{amsmath}
\usepackage[english.slash]{babel} 

\usepackage{times}
\usepackage{url}
\usepackage{graphicx}
\usepackage{latexsym}

\DeclareMathAlphabet{\mathcal}{OMS}{cmsy}{m}{n}

\usepackage{scalefnt}
\usepackage{tikz}
\usetikzlibrary{trees}
\usetikzlibrary{shadows,patterns,shapes}

\usepackage{enumerate}
\usepackage{graphicx}
\usepackage{color}

\usepackage{amsmath}
\usepackage{amsfonts}
\usepackage{amssymb,bm}
\usepackage{verbatim}
\usepackage{microtype}

\newcommand{\bfe}[1]{\begin{bfseries}\emph{#1}\end{bfseries}\index{#1}}

\newcommand{\lra}{\mbox{$\:\leftrightarrow\:$}}

\newcommand{\sse}{\mbox{$\:\subseteq\:$}}

\newcommand{\LL}{\mbox{$\ldots$}}

\newcommand{\NI}{\noindent}
\newcommand{\HB}{\hfill{$\Box$}}

\newcommand{\II}{\vspace{2 mm}}

\newcommand{\szkew}[1]{\relax \setbox0=\hbox{\kern -24pt $\displaystyle#1$\kern 0pt }%
\box0}
{\catcode`\@=11 \global\let\ifjusthvtest@=\iffalse}
\newcounter{oldmycaption}

\newcommand{\weg}[1]{}
\newcommand{\lang}{\mathcal{L}}

\newcommand{\Agents}{\mathsf{A}}

\newcommand{\Atomsb}{\mathsf{Q}}

\newcommand{\lprop}{\mathcal{L}_0}
\newcommand{\lwn}{\mathcal{L}_1}
\newcommand{\la}{\mathcal{L}^a_1}

\newcommand{\NPlog}{\text{P}^{\text{NP}}_{\parallel}}
\newcommand{\NP}{{\sc NP}\xspace}
\newcommand{\coNP}{{\sc co-NP}\xspace}

\newtheorem{theorem}{Theorem}

\newtheorem{lemma}[theorem]{Lemma}
\newtheorem{definition}{Definition}
\newtheorem{corollary}[theorem]{Corollary}
\newtheorem{example}{Example}
\newtheorem{problem}{Problem}
\newtheorem{claim}{Claim}

\newcommand{\qed}{\hfill$\Box$}

\newenvironment{proof}{\NI {\em Proof.}}{\qed}

\newcommand{\set}[1]{{\{ #1 \}}}

\newcommand{\Calls}{\mathsf{C}}
\newcommand{\Situation}{\mathsf{s}}

\newcommand{\init}{\mathsf{root}}
\newcommand{\Call}{\mathsf{c}}
\newcommand{\Calld}{\mathsf{d}}

\newcommand{\CSequences}{\boldsymbol{\Calls}}

\newcommand{\CSequence}{\boldsymbol{\Call}}
\newcommand{\CSequenced}{\boldsymbol{\Calld}}

\title{Open Problems in a Logic of Gossips}
\author{Krzysztof R. Apt
	\institute{Centrum Wiskunde \& Informatica\\ Amsterdam, The Netherlands}
	\institute{University of Warsaw\\ Warsaw, Poland}
    \email{k.r.apt@cwi.nl}
	\and
	Dominik Wojtczak
	\institute{University of Liverpool \\
		Liverpool, UK}
	\email{d.wojtczak@liv.ac.uk}
}
\def\titlerunning{Open Problems in a Logic of Gossips}
\def\authorrunning{K.~R.~Apt \& D. Wojtczak}

\begin{document}

\maketitle

\begin{abstract}
  Gossip protocols are programs used in a setting in which each agent
  holds a secret and the aim is to reach a situation in which all
  agents know all secrets. Such protocols rely on a point-to-point or
  group communication.  Distributed epistemic gossip protocols use
  epistemic formulas in the component programs for the agents. The
  advantage of the use of epistemic logic is that the resulting
  protocols are very concise and amenable for a simple verification.

  Recently, we introduced a natural modal logic that allows one to
  express distributed epistemic gossip protocols and to reason about
  their correctness. We proved that the resulting protocols are
  implementable and that all aspects of their correctness, including
  termination, are decidable. To establish these results we showed
  that both the definition of semantics and of truth of the underlying
  logic are decidable. We also showed that the analogous results hold
  for an extension of this logic with the `common knowledge' operator.

  However, several, often deceptively simple, questions about this
  logic and the corresponding gossip protocols remain open. The
  purpose of this paper is to list and elucidate these questions and
  provide for them an appropriate background information in the form
  of partial of related results.
\end{abstract}
\date{}

\section{Introduction}

Gossip protocols concern a set up in which each agent holds initially
a secret and the aim it to arrive, by means of point-to-point or group
communications (called \emph{calls}), at a situation in which all agents know
each other secrets.  During the calls the agents exchange some,
possibly all, secrets they know.

These protocols were successfully used in a number of
domains, for instance communication networks \cite{HHL88}, computation
of aggregate information \cite{kempe2003gossip}, and data replication
\cite{ladin1992providing}.  For a more recent account see
\cite{HKPRU05} and \cite{KvS07}.

In \cite{ADGH14} a dynamic epistemic logic was introduced in which
gossip protocols could be expressed as formulas. These protocols rely
on agents' knowledge and are distributed, so they are distributed
epistemic gossip protocols.  This means that they can be
seen as special cases of knowledge-based programs introduced in
\cite{FHMV97}.

In \cite{AGH16} a simpler modal logic was introduced that is
sufficient to define these protocols and to reason about their
correctness. This logic is interesting in its own rights and was
subsequently studied in a number of papers. In particular, in
\cite{AW16}, and in the full version in \cite{AW18}, we established
decidability of its semantics and truth for a limited
fragment. Building upon these results we then proved that the
distributed gossip protocols, the guards of which are defined in this
logic, are implementable, that their partial correctness is decidable,
and that termination and two forms of fair termination of these
protocols are decidable, as well.  Further, in \cite{AKW17} the
computational complexity of this fragment was studied and in
\cite{AW17a} we considered its extension with the common knowledge
operator for which we established analogous decidability results.
  
In spite of the simplicity of this modal logic several natural
questions about it and the gossip protocols defined using this logic
remain open. In what follows we discuss these problems. For each
of them we provide the relevant background information and establish
some partial results.

Among these partial results let us mention the following ones:

\begin{itemize}
\item When the agents form a star graph, each correct distributed
  epistemic gossip protocol in the framework of \cite{AGH16} has to
  rely on guards with the modal operators (Theorem \ref{thm:star} in
  Section \ref{sec:gp1}).  This is relevant since the complexity of
  determining truth of guards is higher in presence of modal operators
  (see Section \ref{sec:open}).

\item It is well known (see, e.g., \cite{tijdeman:1971} discussed in
  Section \ref{sec:gp1}) that for 4 agents 4 calls are both needed and
  sufficient to reach a situation in which all agents know all
  secrets. The resulting protocol is centralized.  We show that when
  distributed epistemic gossip protocols are used, for 4 agents 5
  calls are both needed and sufficient (Theorems \ref{thm:priv3} and
  \ref{thm:2n-3} in Section \ref{sec:gp1}).

\item In the literature on distributed computing, e.g., on Calculus of
  Communicating systems (CCS) of \cite{Mil80}, there is a wealth of
  literature on various ways of comparing behaviour of two processes
  (see, e.g., \cite{San09} for an extensive overview of the
  fundamental concept of bisimulation).  Distributed epistemic gossip
  protocols can be naturally compared by means of a notion that one of
  them can simulate another. We show that checking it can be done in
  exponential time (Theorem \ref{thm:simulate} in Section
  \ref{sec:gp2}).

\end{itemize}
Further, the arguments used to prove the last result imply that all 
computations of a terminating gossip protocol are of length 
$< n^4$.

The paper is organized as follows.  In the next section we discuss related
work. Then, in Section \ref{sec:logic-recall}, introduce
the already mentioned logic, originally defined in \cite{AGH16}, and in
Section \ref{sec:open} discuss some natural open problems about it. In
Section \ref{sec:ck} we recall an extension of this logic with the
common knowledge operator and introduce two open problems concerning
it.  Next, in Section \ref{sec:distributed}, we recall the distributed
epistemic gossip protocols considered in \cite{AGH16}. Then, in
Sections \ref{sec:gp1} and \ref{sec:gp2} we discuss natural open
problems about these protocols. For several open problems we provide
some partial results.

\section{Related work}

As already mentioned, distributed epistemic gossip protocols were
introduced in \cite{ADGH14}. In \cite{ADGH14a} a tool was presented
that given a high level description of an epistemic protocol in this
setting generates the characteristics of the protocol. In both papers
three types of calls were considered but the ones considered here
(initially studied in \cite{AGH16}) differ from them in that we assume
that agents not participating in the call are not aware of it.  In
\cite{AGH16} also two other modes of communication were considered,
in which only one agent (the caller or the called one) learns new
secrets.  The assumptions about the calls used in \cite{ADGH14} and
\cite{AGH16} were presented in a uniform framework in \cite{AGH18},
where in total 18 types of communication were introduced and compared
w.r.t.~their epistemic strength.

In \cite{HM15} and \cite{herzig_how_2017} centralized gossip protocols
were studied the aim of which is to achieve higher-order shared
knowledge, for example knowledge of level 2 which stipulates that
everybody knows that everybody knows all secrets. In particular, a
protocol was presented and proved correct that achieves in
$(k+1)(n-2)$ steps shared knowledge of level $k$.  These matters were
further investigated in \cite{cooper_simple_2016}, where optimal
protocols for various versions of such a generalized gossip problem
were presented. These protocols depend on various parameters, for
example type of the underlying graph or the type of
communication. Further, different gossip problems were also studied in
which some negative goals, for example that certain agents must not
know certain secrets, are supposed to be achieved.  In
\cite{cooper2018temporal} such problem were studied further in the
presence of temporal constraints, i.e., a given call can only (or has
to) be made within a given time interval.

Then in \cite{CHMMR16} gossip protocols were analyzed as an instance
of multi-agent epistemic planning that was subsequently translated
into a planning language.

The underlying framework was analyzed from a number of views. In
\cite{van_ditmarsch_parameters_2016} gossip problems were considered
in an epistemic framework that provides several parameters allowing us
to capture various aspects of it, for example the initial knowledge of
the agents, the type of communication used, and the desired type of
the protocol (for example, a symmetric one). For some of the
combinations of the parameters the minimum number of calls needed to
reach the final situation was established.  The expected time of
termination of several gossip protocols for complete graphs was
studied by \cite{DKS17}.

Next, in \cite{DEPRS17} dynamic distributed gossip protocols were
studied in which the calls allow the agents not only to share the
secrets but also to transmit the links. These protocols were
characterized in terms of the class of graphs for which they
terminate. They differ from the ones here considered, which are
static.  This set up was further investigated in \cite{DEPRS18} where
various dynamic gossip protocols were proposed and analyzed.  In
\cite{gattinger2018towards} such protocols were analyzed by embedding
them in a network programming language Net\textsc{KAT} proposed in
\cite{AFGJKSW14}.

\section{Logic: a recall}
\label{sec:logic-recall}

We recall here the framework of \cite{AGH16}.  We assume a fixed set
$\Agents$ of $n \geq 3$ \bfe{agents} and stipulate that each agent
holds exactly one \bfe{secret}, and that there exists a bijection
between the set of agents and the set of secrets. We use it implicitly
by denoting the secret of agent $a$ by $A$, of agent $b$ by $B$, etc.
We denote by $\mathsf{Sec}$ the set of all secrets.

The language $\lang$ is defined by the following grammar:
\[
\phi ::= F_a S \mid \neg \phi \mid \phi \land \phi \mid K_a \phi,
\]
where $S \in \mathsf{Sec}$ and $a \in \Agents$.

So $F_a S$ is an atomic formula, while $K_a \phi$ is a compound
formula. We read $F_a S$ as `agent $a$ is familiar with the secret
$S$' (or `agent $a$ holds secret $S$') and $K_a \phi$ as `agent $a$
knows the formula $\phi$'.  Below we shall freely use other Boolean
connectives that can be defined using $\neg$ and $\land$ in a standard
way.

\newcommand{\Exp}{\text{Exp}}
In the sequel we shall use the following formula
\[
  \Exp_a \equiv \bigwedge_{S \in \mathsf{Sec}} F_a S,
\]
that denotes the fact that agent $a$ is familiar with all the secrets
(is an `expert').

In the paper we shall use the following sublanguages of $\lang$:

\begin{itemize}
\item $\lprop$, its propositional part, which consists of the formulas
that do not use the $K_a$ modalities; 

\item $\lwn$,  which consists of the formulas
without the nested use of the $K_a$ modalities;

\item $\la$, where $a \in \Agents$ is a fixed agent,
which consists of the formulas from $\lwn$ where the only modality is $K_a$.

\end{itemize}

Each \bfe{call}, written as $ab$, concerns two different agents, the \bfe{caller},
$a$, and the \bfe{callee}, $b$.  After the call the caller and the
callee learn each others secrets.  Calls are denoted by $\Call$,
$\Calld$.  Abusing notation we write $a \in \Call$ to denote that
agent $a$ is one of the two agents involved in the call $\Call$.
 
Following \cite{AGH16} we stipulate that agents not involved in the
call are not aware of it.  This will be addressed in Definition
\ref{def:model} below.

In what follows we focus on call sequences. Unless explicitly stated
each call sequence is assumed to be finite.  The empty sequence is
denoted by $\epsilon$.  We use $\CSequence$ to denote a call sequence
and $\CSequences$ to denote the set of all finite call sequences.
Given call sequences $\CSequence$ and $\CSequenced$ and a call $\Call$
we denote by $\CSequence.\Call$ the outcome of adding $\Call$ at the
end of the sequence $\CSequence$ and by $\CSequence.\CSequenced$ the
outcome of appending the sequences $\CSequence$ and $\CSequenced$.  We
say that $\CSequence_2$ is a \bfe{subsequence} of a call sequence
$\CSequence$ if for some call sequences $\CSequence_1$ and
$\CSequence_3$ we have $\CSequence = \CSequence_1.\CSequence_2.\CSequence_3$.

To describe what secrets the agents are familiar with, we use the
concept of a \bfe{gossip situation}. It is a sequence
$\Situation = (\Atomsb_a)_{a \in \Agents}$, where
$\{A\} \sse \Atomsb_a \sse \mathsf{Sec}$ for each agent $a$.  Intuitively, $\Atomsb_a$
is the set of secrets $a$ is familiar with in the gossip situation
$\Situation$.  The \bfe{initial gossip situation} is the one in which
each $\Atomsb_a$ equals ${\{A\}}$ and is denoted by $\init$.  It
reflects the fact that initially each agent $a$ is familiar only with his
own secret, $A$.  We say that an agent $a$ is an \bfe{expert} in a gossip
situation $\Situation$ if $a$ is familiar in $\Situation$ with all the
secrets, i.e., if $\Atomsb_a = \mathsf{Sec}$.

Each call transforms the current gossip situation 
by modifying the sets of secrets the agents involved in the call are
familiar with as follows.  Consider a gossip situation
$\Situation := (\Atomsb_d)_{d \in \Agents}$ and a call $ab$.

Then
\[
  ab(\Situation) := (\Atomsb'_d)_{d \in \Agents},
\]
where $\Atomsb'_a = \Atomsb'_b = \Atomsb_a \cup \Atomsb_b$, and for
$c \not \in \{a,b\}$, $\Atomsb'_c = \Atomsb_c$.

So the effect of a call is that the caller and the callee share the secrets
they are familiar with.

The result of applying a call sequence to a gossip situation $\Situation$ is
defined inductively as follows:
\[
\epsilon(\Situation) := \Situation, \
(\Call.\CSequence)(\Situation) := \CSequence(\Call(\Situation)).
\]

\begin{example} \label{exa:1}
\rm
We will use the following concise notation for gossip situations. Sets
of secrets will be written down as lists. E.g., the set
$\set{A, B, C}$ will be written as $ABC$. Gossip situations will be
written down as lists of lists of secrets separated by dots. E.g., if
there are three agents, $a$, $b$ and $c$, then $\init = A.B.C$ and the
gossip situation $(\set{A,B}, \set{A,B}, \set{C}$) will be written as
$AB.AB.C$.

Let $\Agents = \set{a,b,c}$. 
Consider the call sequence $(ac,bc,ac)$. It generates
the following successive gossip situations starting from $\init$:
\[
A.B.C \stackrel{ac}{\longrightarrow} AC.B.AC \stackrel{bc}{\longrightarrow} AC.ABC.ABC
\stackrel{ac}{\longrightarrow} ABC.ABC.ABC.
\]
Hence $(ac,bc, ac)(\init) = (ABC.ABC.ABC)$.
\HB
\end{example}

As calls progress in sequence from the initial situation, agents may
be uncertain about which call sequence took place.  This uncertainty
is captured by the appropriate equivalence relations on the call
sequences.  

\begin{definition} \label{def:model}
  Fix an agent $a$.  We define
  $\sim_a \subseteq \CSequences \times \CSequences$ as the smallest
  equivalence relation satisfying the following conditions:

\begin{description}
\item{{\em [Base]}} $\epsilon \sim_a \epsilon$,

\item{{\em [Step]}} Suppose $\CSequence \sim_a \CSequenced$.

  \begin{enumerate}[(i)]

  \item If $a \not\in \Call$, then $\CSequence.\Call \sim_a \CSequenced$
and $\CSequence \sim_a \CSequenced.\Call$.

\item If there exists $b \in \Agents$ and $\Call \in \set{ab, ba}$ 
such that  $\CSequence.\Call(\init)_a = \CSequenced.\Call(\init)_a$, then
$\CSequence.\Call \sim_a \CSequenced.\Call$.

  \end{enumerate}
\end{description}
\end{definition}

In \emph{(i)} we formalize the assumption that the agents are not
aware of the calls they do not participate in.  In turn, in
\emph{(ii)} we capture the intuition that two call sequences are
indistinguishable for an agent
if the sets of his calls in both sequences are the same
and in each sequence he observes the same set of secrets.
For instance, by \emph{(i)} we have $ab, bc \sim_a ab, bd$.
But we do not have $bc, ab \sim_a bd, ab$ since
$C \in (bc, ab)(\init)_a$, while $C \not\in (bd, ab)(\init)_a$.

Next, we recall the definition of truth.

\begin{definition} \label{def:semantics}%
  Consider 
  a call sequence $\CSequence \in \CSequences$. We define the
  satisfaction relation $\models$ inductively as follows (clauses for
  Boolean connectives are as usual and omitted):
\begin{eqnarray*}
\CSequence \models F_a S & \mbox{iff} & S \in \CSequence(\init)_a, \\
\CSequence \models K_a \phi &  \mbox{iff}  & \,\forall \CSequenced \mbox{     s.t.     } \CSequence \sim_a \CSequenced, ~\CSequenced \models \phi. 
\end{eqnarray*}
Further, we say that $\phi$ is \textbf{true}, and write
$\models \phi$,  when $\forall \CSequence \ \CSequence \models \phi$.
Also, we say that $\phi$ and $\psi$ are \textbf{equivalent} if
$\forall \CSequence \ \CSequence \models \phi \lra \psi$.
\end{definition}

So a formula $F_a S$ is true after the call sequence $\CSequence$
whenever secret $S$ belongs to the set of secrets agent $a$ is
familiar with in the situation generated by the call sequence
$\CSequence$ applied to the initial situation $\init$.
Hence $\CSequence \models Exp_a$ iff agent $a$ is an expert in 
$\CSequence(\init)$.

The knowledge operator is interpreted as customary in epistemic logic,
using the equivalence relations $\sim_a$.

\section{Open problems about the logic}
\label{sec:open}

The first problem concerns the propositional part $\lprop$ of the
language $\lang$.  In \cite{AKW17} we established that the problem of
determining the complexity of the semantics of the language of $\lprop$
(i.e, determining whether $\CSequence \models \phi$ for a given call
sequence $\CSequence$) is in P, while the problem of determining truth
in the language of $\lprop$ (i.e, determining whether
$\CSequence \models \phi$ for all call sequences $\CSequence$) is
\coNP-complete.  Note that the latter implies that checking whether
there exists a call sequence $\CSequence$ such that
$\CSequence \models \phi$ is \NP-complete.

\begin{problem} \label{pro:1}
  Find an axiomatization of $\lprop$.
\end{problem}

\NI
\textbf{Comments}

We stipulate that the following formula should be an axiom:
\begin{equation}
  \label{equ:2}
  \bigwedge_{\substack{a,b \in \Agents \\ a \neq b}} (F_a B \to \bigvee_{k = 2}^{n} \bigvee_{\substack{a_1, \ldots,
      a_k \in \Agents \\ a_1 = b, a_k = a \\ a_i \neq a_j \text{ for all } i \neq j}} \bigwedge_{h = 1}^{k-1} (F_{a_h} A_{h+1} \land F_{a_{h+1}} A_{h})).
\end{equation}

It formalizes the fact that communication is bidirectional and that
agents learn secrets through a chain of calls, and is easily seen to be
true.  Here and elsewhere we make use of the fact that there is a
bijection between secrets and agents, which allows us to use agents
(here $b, a_h$ and $a_{h+1}$) when referring to the corresponding
secrets (here $B, A_h$ and $A_{h+1}$).  Note that the second
disjunction is finite, since we postulate that the agents
$a_1, \LL, a_k$ are pairwise different.

Here is the version of (\ref{equ:2}) when $\Agents$ has three agents that is more readable:
\[
  \bigwedge_{\substack{a, b, c \in \Agents \\ \{a, b, c\} = \Agents}} (F_a B \to (F_b A \land F_a B) \lor (F_b C \land F_c B \land F_c A \land F_a C)).
\]

Assuming a sound and complete proof system for Boolean formulas formula
(\ref{equ:2}) implies the following natural formula of $\lprop$
stating that each agent can learn a new secret only by revealing its
own:

\begin{equation}
  \label{equ:1}
\bigwedge_{\substack{a,b \in \Agents \\ a \neq b}} (F_a B \to \bigvee_{\substack{c,a \in \Agents \\ c \neq a}} F_c A).
\end{equation}

To see the claim it suffices to note that (\ref{equ:2}) implies 
\[
  \bigwedge_{\substack{a,b \in \Agents \\ a \neq b}} (F_a B \to \bigvee_{k = 2}^{n} \bigvee_{\substack{a_1, \ldots,
          a_k \in \Agents \\ a_1 = b, a_k = a \\ a_i \neq a_j \text{ for all } i \neq j}} F_{a_{k-1}} A),
\]
from which (\ref{equ:1}) follows.

In \cite{AGH18} it was observed that the following formula is true:
\begin{equation}
  \label{equ:3}
  \bigwedge_{\substack{a,b \in \Agents \\ a \neq b}} \Big(F_a B \land \bigwedge_{\substack{i \in \Agents \\ i \neq a, i \neq b}} \neg F_i B\Big) \to F_b A.
\end{equation}

It states that if agent $a$ is the only agent (different from $b$) familiar
with the secret of $b$, then agent $b$ is familiar with the secret of $a$.

It is easy to see that
assuming a sound and complete proof system for Boolean formulas
(\ref{equ:2}) implies (\ref{equ:3}). To this end it suffices to note that
(\ref{equ:2}) implies
\[
  \bigwedge_{\substack{a,b \in \Agents \\ a \neq b}} (F_a B \to F_b A \lor \bigvee_{k =
    3}^{n} \bigvee_{\substack{a_1, \ldots,
        a_k \in \Agents \\ a_1 = b, a_k = a \\ a_i \neq a_j \text{ for all } i \neq j}} F_{a_{2}} B),
\]
from which (\ref{equ:3}) follows.

A more ambitious problem is the following one.

\begin{problem} \label{pro:2}
  Find an axiomatization of $\lang$.
  Determine the complexity of the semantics and of truth in the language
of $\lang$.
\end{problem}

\NI
\textbf{Comments}

In \cite{AW16} and in the full version in \cite{AW18}, we showed that
for the sublanguage $\lwn$ of $\lang$ both the semantics and truth are
decidable.  In \cite{AKW17} we sharpened these results by showing that
the first problem is $\NPlog$-complete, while the complexity of
determining the truth is in $\text{coNP}^{\text{NP}}$. 
It is not clear how to extend these results to larger fragments of $\lang$.
Recall that the complexity class $\NPlog$, defined in \cite{wagner1987more},
corresponds
to the class of problems solvable by 
a deterministic polynomial-time Turing machine that has a parallel 
access to an NP oracle (i.e., no query to this oracle can depend on the outcome of any other). 
In
\cite{wagner1987more,spakowski2000theta} many natural problems were shown to be complete for this class.
One of them is checking for two Boolean formulas $\phi,\phi'$ 
whether the maximum number of variables assigned {\bf true} in a satisfying assignment
is greater for $\phi$ than for $\phi'$.
On the other hand, the complexity class ${\text{coNP}}^{\text{NP}}$, commonly denoted by $\Pi^P_2$,
corresponds to a non-deterministic polynomial-time
Turing machines with an access to an \NP oracle that accepts a given
input if and only if all its non-deterministic branches accept it. An example
problem complete for this class is the satisfiability for quantified
Boolean formulas with two alternations of quantifiers
$\forall {x_1,\ldots,x_k} \exists {y_1,\ldots,y_l} \phi$, where $\phi$
is a Boolean formula over the variables $x_1,\ldots,x_k,y_1,\ldots,y_l$.

\section{Open problems concerning common knowledge}
\label{sec:ck}

In \cite{AW17a} an extension of the language $\lang$ was considered that involves
the common knowledge operator.
The resulting language $\lang_{ck}$ is
defined by the following grammar:
\[
\phi ::= F_a S \mid \neg \phi \mid \phi \land \phi \mid C_G \phi,
\]
where $S \in \mathsf{Sec}$ and $a \in \Agents$ and $G \sse \Agents$.

Recall that the semantics of the $C_G$ operator is defined as follows.
Given a set $A$ let $A^*$ be the set of all finite sequences formed
from the elements of $A$.  For $t = a_1, \LL, a_k$ let
$K_t = K_{a_1} \LL K_{a_k}$.  Then we stipulate that
\[
  C_G \phi \equiv \bigwedge_{t \in G^*} K_t \phi.
\]
Note that the formula on the right-hand side is an infinite conjunction.  

When $G$ is a singleton, say $G = \{a\}$, the formula $C_G \phi$ has
the same semantics as $K_a \phi$, since $K_a K_a \phi$ is equivalent
to $K_a \phi$.  So the language $\lang_{ck}$ can be viewed as an
extension of $\lang$.

\begin{problem} \label{pro:ck}
  Determine whether common knowledge is equivalent to a nested knowledge.
  More precisely, determine whether for every $G$ and $\phi$ there exists $t \in G^*$ such
  that the formulas $C_G \phi$ and $K_t \phi$ are equivalent.
\end{problem}

If the answer to Problem \ref{pro:ck} is positive, then it is natural
to ask whether $t \in G^*$ can be found independently of the formula
$\phi$, that is, whether for some $t \in G^*$ the equivalence
\[
  C_{G} \phi \lra K_t \phi
\]
holds for all formulas $\phi$ in $\lang_{ck}$ or in some fragment of it.

\begin{problem} \label{pro:4}
  Find an axiomatization of $\lang_{ck}$.
  Determine the complexity of the semantics and of truth in the language
of $\lang_{ck}$.
\end{problem}

\NI
\textbf{Comments}

We do not know the answer to Problem \ref{pro:ck} even for the formulas of the form
$C_{G} \phi$, where $\phi \in \lprop$. In \cite{AW17} we succeeded to establish
the following limited results.
\II

\NI
{\scshape Theorem}

\begin{enumerate}[(i)]
\item
\emph{Suppose that $|G| \geq 3$. Then for all call sequences $\CSequence$ and formulas $\phi \in \lang_{ck}$
\[
\CSequence \models C_G \phi \mbox{ iff } \models \phi.
\]
Consequently, the formulas $C_G \phi$ and $\phi$ are equivalent.
}
\item
  \emph{For all call sequences $\CSequence$ that do not
contain the call $ab$ and all formulas $\phi \in \lang_{ck}$ that do
not contain the $\neg$ symbol
\[
\CSequence \models C_{\{a,b\}} \phi \mbox{ iff } \CSequence \models K_{abab} \phi.
\]
}
\end{enumerate}

Part (i) states that the formulas commonly known by a group of at
least three agents are precisely the true formulas. Part (ii) states
that for a group of two agents common knowledge of negation-free formulas 
is equivalent to the 4th fold iterated knowledge w.r.t.~a call
sequence in which no calls between these two agents were made.
Examples in \cite{AW17} show that this claim does not hold for arbitrary formulas and
that the restriction on the call sequence cannot be dropped.

In \cite{AW17} we showed that both the semantics and truth in the
sublanguage of $\lang_{ck}$ that consists of the formulas with no nested
$C_G$ modalities are decidable.

\section{Distributed epistemic gossip protocols: a recall}
\label{sec:distributed}

In \cite{AGH16}, as a follow up on \cite{ADGH14}, we studied
distributed epistemic gossip protocols.  Their goal is to reach a
gossip situation in which each agent is an expert.  In other words,
their goal is to transform a gossip situation in which the formula
$\bigwedge_{a \in \Agents} (F_a A \land \bigwedge_{b \in
  \Agents, b \neq a} \neg F_a B)$ is true into one in which the
formula $\bigwedge_{a, b \in \Agents} F_a B$ is true. As explained
in the introduction, in \cite{AGH16} a different syntax of the gossip
protocols than in \cite{ADGH14} was used.

Let us recall the definition.  The adopted syntax follows the syntax
of the CSP language (Communicating Sequential Processes) of
\cite{Hoa78}, in which '$*$' denotes a repetition and `$[]$' a
nondeterministic choice.  By a \bfe{component program}, in short a
\bfe{program}, for an agent $a$ we mean a statement of the form
\[ 
*[[]^m_{j=1}\ \psi_j \to \Call_j],
\]
where $m \geq 0$ and each $\psi_j \to \Call_j$ is such that

\begin{itemize}
\item $a$ is the caller in the call $\Call_j$,

\item $\psi_j \in \la$ and all atomic formulas used in $\psi$ start with $F_a$. 

 \end{itemize}
If $m = 0$, the program is empty.

We call each such
construct $\psi \to \Call$ a \bfe{rule} and refer in this context to
$\psi$ as a \bfe{guard}.
Intuitively, $*$ denotes a repeated execution of the rules, one at a time, where
each time non-deterministically a rule is selected whose guard is true.

By a \bfe{distributed epistemic gossip protocol}, from now on just a
\bfe{gossip protocol}, we mean a parallel composition of component
programs, one for each agent.  We call a gossip protocol
\bfe{propositional} if all guards in it are propositional, i.e., are
from the language $\lang_{p}$.

We presuppose that in each gossip protocol the agents are the nodes of
a directed graph (digraph) and that each call $ab$ is allowed only if $a \to b$
is an edge in the digraph.  A minimal digraph that satisfies this
assumption is uniquely determined by the syntax of the protocol. Given
that the aim of each gossip protocol is that all agents become experts
it is natural to assume that this digraph is connected.

Here are two examples of gossip protocols to which we shall return later.
\begin{example}\label{exa:lns}
  In \cite{ADGH14} the following correct gossip protocol, called
  \emph{Learn New Secrets} (LNS in short), for complete graphs was
  proposed. In the syntax of \cite{AGH16} used here it is
  propositional, as it has the following program for agent $i$:
\[ 
*[[]_{j \in \Agents} \neg F_i J \to ij].
\]
Informally,  agent $i$ calls agent $j$ if agent $i$ is not familiar with $j$'s secret.
\HB
\end{example}

\begin{example}\label{exa:hms}
In \cite{ADGH14} also the following correct gossip protocol, called \emph{Hear My
  Secret} (HMS in short), for complete graphs was proposed. In the
syntax of \cite{AGH16} it has the following program for agent
$i$:

\[
*[[]_{j \in \Agents} \neg K_i F_j I \to ij].
\]
Informally, agent $i$ calls agent $j$ if agent $i$ does not know whether $j$
is familiar with his secret. 
\HB
\end{example}

Consider a gossip protocol $P$ that is a parallel composition of the component programs
$* [[]^{m_a}_{j=1}\ \psi^a_j \to \Call^a_j]$, one for each agent $a \in \Agents$.

The \bfe{computation tree} %
of $P$ is a directed tree
defined inductively as follows.  Its nodes are call sequences and its
root is the empty call sequence $\epsilon$.  Further, if $\CSequence$
is a node and for some rule $\psi^a_j \to \Call^a_j$ we have
$\CSequence \models \psi^a_j$, then $\CSequence.\Call^a_j$
is a node that is a direct descendant of $\CSequence$. Intuitively,
the arc from $\CSequence$ to $\CSequence.\Call^a_j$ records the effect
of the execution of the rule $\psi^a_j \to \Call^a_j$ performed after the call
sequence $\CSequence$ took place.

By a \bfe{computation} of a gossip protocol we mean a maximal rooted
path in its computation tree.  In what follows we identify each computation with
the unique call sequence it generates.
We say that the gossip protocol $P$ is \bfe{partially correct}
if for all leafs $\CSequence$ of the computation tree of $P$
\begin{equation}
  \label{cond:expert}
\CSequence \models \bigwedge_{a \in \Agents, S \in \mathsf{Sec}} F_a S,
\end{equation}
i.e., if each agent is an expert in the gossip situation
$\CSequence(\init)$.

We say furthermore that $P$ \bfe{terminates} if all its computations
are finite and say that $P$ \bfe{is correct} if it is partially correct and
terminates. 

We also consider two variants of termination. To define
them we need a subsidiary notion.  We call a rule \bfe{enabled} after
a call sequence $\CSequence$ if its guard is true after $\CSequence$.
Given a gossip protocol we say that an agent is \bfe{enabled} after a
call sequence $\CSequence$ if one of the rules in its program is
enabled.

We now stipulate that each finite computation is \bfe{rule-fair} and
\bfe{agent-fair}.  An infinite computation is \bfe{rule-fair}
(resp.~\bfe{agent-fair}) if all rules (resp.~agents) that are enabled
after infinitely many prefixes (in short, infinitely often) are
selected infinitely often.  We say that a gossip protocol $P$
\bfe{rule-fairly terminates} (resp.~\bfe{agent-fairly terminates}) if
all its rule-fair (resp.~agent-fair computations) are finite.
Agent-fairness was introduced in \cite{AGH16}, where it was simply
called fairness, while rule-fairness was introduced in \cite{AW17} and
further studied in \cite{AW18}.

\section{Open problems about the gossip protocols}
\label{sec:gp1}

We begin with the following problem.

\begin{problem}
  Characterize the class of graphs for which 
  correct propositional gossip protocols exist.
\end{problem}

Note that the LNS protocol from Example \ref{exa:lns} shows that
this class of graphs include all complete digraphs.
However, as we will show in Theorem \ref{thm:star}, 
star graphs do not belong to this class.
We conjecture that any digraph whose complement of its set of edges
contains a directed cycle does not have this property.

\NI
\textbf{Comments}

For further discussion it is useful to introduce some terminology.  We
say that a gossip protocol is \bfe{for arbitrary connected digraphs}
if each agent $a$ can only call the agents from the set $N_a$ of
its in-neighbours. In such gossip protocols the underlying digraph is
a parameter that can be uncovered from the sets $N_a$.  If the
underlying digraph is supposed to be complete, we say that the gossip
protocol is \bfe{parametric}.  So in both cases we actually deal with
a `parametrized' gossip protocol, which is a template for an infinite
set of gossip protocols.

For example, in \cite{AW17} and \cite{AW18} a correct gossip protocol
with the following program for agent $i$ was considered:
\[
*[[]_{j \in N_i, S \in \mathsf{Sec}} F_i S \wedge \neg K_i F_j S \to ij].
\]
Informally, agent $i$ calls a neighbour $j$ if $i$ is familiar with
some secret (here $S$) and he does not know whether $j$ is familiar
with it. This gossip protocol is for arbitrary connected digraphs.
However, it is not propositional.

In Example \ref{exa:lns} the LNS gossip protocol was introduced. It is
parametric and is both correct and propositional.  However, its
counterpart for arbitrary connected digraphs, so with the program
\[ 
*[[]_{j \in N_i} \neg F_i J \to ij].
\]
for agent $i$, is obviously incorrect.
Indeed, for the graph%
\smallskip

\begin{center}
\begin{tikzpicture}

  \coordinate [label={below: $i$}] (A) at (0, 0);
  \coordinate [label={below:$j$}] (B) at (1, 0);
  \coordinate [label={below:$k$}] (C) at (2, 0);
  \fill[color=black] (A) circle (0.07cm);
  \fill[color=black] (B) circle (0.07cm);
  \fill[color=black] (C) circle (0.07cm);

  \draw [thick] (A) -- (B) -- (C);
\end{tikzpicture}
\end{center}

\NI
it terminates after the calls $ij$, $jk$ with the agent $i$ not being an expert.

In turn, a natural gossip protocol \emph{Exp} for arbitrary connected
digraphs, with the program
\[ 
*[[]_{j \in N_i} \neg \Exp_i \to ij]
\]
for agent $i$, is obviously partially correct but it does not
terminate, as initially a fixed call $ij$, with $j \in N_i$, can be
repeated indefinitely.

In \cite{AW18} we proved that the \emph{Exp} gossip protocol agent-fairly
terminates in the case of rings.   However, this is not the case in
general. Indeed take the graph

\smallskip

\begin{center}
\begin{tikzpicture}

  \coordinate [label={below: $i$}] (A) at (0, 0);
  \coordinate [label={below:$j$}] (B) at (1, 0);
  \coordinate [label={below:$k$}] (C) at (2, 0);
  \coordinate [label={below:$l$}] (D) at (3, 0);
  \fill[color=black] (A) circle (0.07cm);
  \fill[color=black] (B) circle (0.07cm);
  \fill[color=black] (C) circle (0.07cm);
  \fill[color=black] (D) circle (0.07cm);

  \draw [thick] (A) -- (B) -- (C) -- (D);
\end{tikzpicture}
\end{center}

\NI
Then $(ij, ji, kl, lk)^*$ is an infinite agent-fair computation.

On the other hand, the following result holds.  It generalizes the
above result of \cite{AW18} since for the case of rings the program
for each agent has just one guard and consequently the notions of
agent-fairness and rule-fairness coincide.

\begin{theorem}
  The gossip protocol $Exp$ for arbitrary connected digraphs
 rule-fairly terminates.
\end{theorem}

\begin{proof}
Suppose otherwise. Consider an infinite rule-fair computation
$\xi$. We say that an agent $i$ \emph{becomes an expert in $\xi$} if
for some element $\CSequence$ of $\xi$ we have
$\CSequence \models \Exp_i$.

As some agent $i$ does not become an expert in $\xi$, there is a
secret $J$ that $i$ does not learn in $\xi$. Let
$i = i_1, i_2, \LL, i_h = j$ be a path connecting agent $i$ with $j$.
By rule-fairness the call $i_{1} i_{2}$ takes place infinitely
often. Hence agent $i_{2}$ does not become an expert in $\xi$ and
consequently by rule-fairness the call $i_{2} i_{3}$ takes place
infinitely often. Repeating this argument we conclude that each call
$i_{g} i_{g+1}$, where $g \in \{1, \LL, h-1\}$, takes place infinitely
often in $\xi$. So the call sequence
$i_{h-1} i_{h}, i_{h-2} i_{h-1}, \LL, i_{1} i_{2}$, possibly
interspersed with other calls, exists in $\xi$. After the last call
agent $i$ learns the secret $J$, which yields a contradiction.
\end{proof}

We now show that for a natural class of connected graphs no correct
propositional gossip protocol exists.

\begin{theorem} \label{thm:star}
  Suppose that the agents form a star
  graph, so a graph in which some agent, say $a$, is present in all
  edges. Then no correct propositional gossip protocol exists.
\end{theorem}

\begin{proof}
  We begin by making two simple observations.
  
  \begin{claim} \label{cla:2} Consider a propositional gossip protocol
    $P$.  Suppose that $\CSequence, \Call, \CSequenced, \Calld$ is a
    prefix of a computation of $P$ such that some agent $a$ 
    \begin{itemize}
    \item is
      involved in the call $\Call$,
      
    \item is not involved in any call in
      $\CSequenced$, and
      
    \item is a caller in $\Calld$.
  \end{itemize}
  Then also
    $\CSequence, \Call, \Calld$ is a prefix of a computation of $P$.
  \end{claim}
  \begin{proof}
    Consider the guard $\phi$ associated with the call $\Calld$.  By
    assumption on $P$, $\phi$ is a propositional formula built out of the atomic
    formulas of the form $F_a S$. By assumption on $\CSequenced$ the
    truth of these atomic formulas is not affected by any call in
    $\CSequenced$. So the truth value of $\phi$ before and after the call
    sequence $\CSequenced$ is the same. But $\phi$ is true just after this
    call sequence, so it is also true just before it.
    This shows that $\Calld$ can be performed immediately after $\Call$.
  \end{proof}

  \begin{claim} \label{cla:1} Let $\xi$ be a computation of a
    terminating propositional gossip protocol. If an agent, say $a$,
    becomes an expert after the call $\Call$, then $a$ is not a caller
    in any call that follows $\Call$ in $\xi$.
  \end{claim}
  \begin{proof}
    Suppose otherwise. Let $\Calld$ be such a call in $\xi$ and let
    $\phi$ be the corresponding guard. This call does not affect the
    set of gossips agent $a$ is familiar with since this set as of the
    call $\Call$ equals $\mathsf{Sec}$. By assumption $\phi$ is a
    propositional formula built out of the atomic formulas of the form
    $F_a S$, so after the call $\Calld$ the formula $\phi$ remains true.  Hence
    the call $\Calld$ can be indefinitely repeated, which shows that
    the considered protocol does not terminate.
  \end{proof}
\II

Let now $G$ be the considered star graph with an agent $a$ present in all
edges.  Suppose by contradiction that a correct propositional gossip
protocol $P$ for $G$ exists. Let $\xi$ be a computation of $P$.  Then
agent $a$ is involved in all calls in $\xi$. Let $\Call'$ be the call
after which $a$ became an expert and let $\Call$ be the call that
precedes $\Call'$ in $\xi$.  By assumption the call $\Call$ concerns
agent $a$ and some agent $b$ not involved in the call $\Call'$.  After
the call $\Call'$ agent $b$ is not yet an expert, hence it is involved
in $\xi$ in another call. Let $\Calld$ be the first such call.

So for some call sequences $\CSequence$ and $\CSequenced$, we have
that $\CSequence, \Call, \CSequenced, \Calld$ is a prefix of
$\xi$. Agent $b$ is not involved in any calls in $\CSequenced$ and by
Claim \ref{cla:1} it is also the caller in $\Calld$.  So by Claim
\ref{cla:2} also $\CSequence, \Call, \Calld$ is a prefix of a
computation, say $\chi$, of $P$. Since both $\Call$ and $\Calld$
involve the same pair of agents, after the second call the set of
gossips of agent $b$ does not change.  Hence in $\chi$ the guard
$\phi$ remains true after $\Calld$ and consequently this call can be
indefinitely repeated. So $P$ does not terminate.
\end{proof}

One of the early results, see for instance \cite{tijdeman:1971}, is
that for $n \geq 4$ agents at least $2n-4$ phone calls are needed to
reach a situation in which each agent is an expert. Further, it is
easy and well-known that this final situation can be reached using $2n-4$ calls.
Indeed, assume that the set of agents is
$\{a, b, c, d, i_1, \LL, i_{n-4}\}$, where $n \geq 4$, (if $n = 4$
then there are no $i_j$ agents) and take the following call sequence
\[
  \begin{array}{l}
  (a, i_1), (a, i_2), \LL, (a, i_{n-4}), \\
  (a,b), (c,d), (a,d), (b,c),  \\
  (a, i_1), (a, i_2), \LL, (a, i_{n-4}).
  \end{array}
\]

\begin{problem} \label{pro:bound}
  Prove that the lower bound $2n-4$ cannot be achieved for the gossip
  protocols.  In other words, prove that every correct gossip
  protocol for $n \geq 4$ agents generates computations of length
  $> 2n- 4$.
\end{problem}

\NI
\textbf{Comments}

We show that this problem can be solved for $n=4$.
To prove it we need the following observations.

\begin{lemma} \label{lem:1}
For all agents
$a, b, c$ and all call sequences $\CSequence$ and all formulas $\phi$
\[
\mbox{$\CSequence \models K_a \phi$ iff $\CSequence, bc \models K_a \phi$.}
\]
Consequently, the same equivalence holds for all formulas that are
Boolean combinations of formulas of the form $K_a \phi$, so in
particular for each guard $\psi$ used in a program for agent $a$.
\end{lemma}

\begin{proof}
By definition $\CSequence \sim_a \CSequence, bc$, which implies the claim.
\end{proof}
\II

This note states that the calls in which agent $a$ is not involved
have no effect on the truth of the guards used in the programs for
agent $a$. This clarifies the syntax of the guards.  If we allowed in
guards for agent $a$ formulas of the form $F_b C$ as conjuncts, this
natural and desired property would not hold anymore.  Indeed, for
$\CSequence = \epsilon$ we have $\CSequence, bc \models F_b C$, while
$\CSequence \not\models F_b C$.

We also need the following observation that confirms the intuition that
two calls involving different pairs of agents can be executed in an arbitrary order.

\begin{lemma} \label{lem:guard}
Consider a protocol $P$. Let $a, b, c, d$ be four agents such that
for some call sequences $\CSequence$ and $\CSequenced$ we have that
$\CSequence, ab, cd, \CSequenced$ is a computation of $P$. Then also
$\CSequence, cd, ab, \CSequenced$ is a computation of $P$. 
\end{lemma}

\begin{proof}
  Let $\phi_a$ be the guard that precedes the call $ab$ in the program
  for agent $a$ and $\phi_c$ the guard that precedes the call $cd$ in the
  program for agent $c$.  We have $\CSequence \models \phi_a$ and
  $\CSequence, ab \models \phi_c$.
By Lemma \ref{lem:1}
\[
\mbox{$\CSequence \models \phi_c$ iff $\CSequence, ab \models \phi_c$}
\]
and 
\[
\mbox{$\CSequence, cd \models \phi_a$ iff $\CSequence \models \phi_a$.}
\]
Hence $\CSequence \models \phi_c$ and $\CSequence, cd \models \phi_a$.
This means that $\CSequence, cd, ab$ is a prefix of a computation of
$P$.

Further, by the definition of the $\sim_i$ relations, for all agents
$i$ and all call sequences $\CSequenced'$ we have
$\CSequence, ab, cd, \CSequenced' \sim_i \CSequence, cd, ab,
\CSequenced'$.  Hence for all formulas of the form $K_i \phi$ we have
\[
\mbox{$\CSequence, ab, cd, \CSequenced' \models K_i \phi$ iff 
$\CSequence, cd, ab, \CSequenced' \models K_i \phi$}
\]
and consequently for all guards $\phi$ preceding the calls in $\CSequenced$
we have
\[
\mbox{$\CSequence, ab, cd, \CSequenced \models \phi$ iff 
$\CSequence, cd, ab, \CSequenced \models \phi$.}
\]
This concludes the proof.
\end{proof}
\II

The above observation also holds for infinite computations but we shall not need it in the
sequel.

Finally, we need to reason about the following notion. We call a computation of a gossip protocol \bfe{bad} if some agent is involved in the
first two calls of it. 

\begin{lemma} \label{lem:bad}
  Every gossip protocol admits bad computations.
\end{lemma}

\begin{proof}
  Fix a gossip protocol $P$ and consider its computation $\xi$ that is
  not bad. By appropriate renaming we can assume that it begins with
  $ab, cd$.  Let $\Call$ be the third call in $\xi$ and $\psi$ the
  guard associated in $P$ with the call $\Call$.  By assumption we
  have $(ab, cd) \models \psi$.  Four cases arise.  \II

\NI
\emph{Case 1}. Agent $a$ is the caller in $\Call$. 

We have $ab, cd \sim_a ab$, so by Lemma \ref{lem:guard}
\[
\mbox{$(ab, cd) \models \psi$ iff $ab \models \psi$}
\]
and hence $ab \models \psi$. In other words, the call $\Call$ can be
executed in $P$ directly after the call $ab$, that is, $ab, \Call$ is a
prefix of a computation of the protocol $P$.
\II

\NI
\emph{Case 2}. Agent $b$ is the caller in $\Call$. 

We have $ab, cd \sim_b ab$, so, as in Case 1,
$ab \models \psi$. Hence the call $\Call$ can be executed in $P$
directly after the call $ab$, that is, $ab, \Call$ is a prefix of a
computation of the protocol $P$.  
\II

\NI
\emph{Case 3}. Agent $c$ is the caller in $\Call$. 

By Lemma \ref{lem:guard} the calls $ab$ and $cd$ can be reversed in $\xi$,
hence a computation of $P$ exists that begins with $cd, ab, \Call$. So
$(cd, ab) \models \psi$.  We have $cd, ab \sim_c cd$, so, as in Case
1, $cd \models \psi$. Hence the call $\Call$ can be executed in $P$
directly after the call $cd$, that is, $cd, \Call$ is a prefix of a
computation of the protocol $P$.
\II

\NI
\emph{Case 4}. Agent $d$ is the caller in $\Call$. 

As in Case 3 a computation of $P$ exists that begins with
$cd, ab, \Call$.  We have $(cd, ab) \models \psi$ and
$cd, ab \sim_d cd$, so, as in Case 3, $cd \models \psi$. Hence the
call $\Call$ can be executed in $P$ directly after the call $cd$, that
is, $cd, \Call$ is a prefix of a computation of the protocol $P$.
\end{proof}

\begin{theorem} \label{thm:priv3} 
Every correct gossip protocol for 4
  agents generates computations of length $> 4$.
\end{theorem}

\begin{proof}
Suppose that $\Agents = \{a, b, c, d\}$.  Consider a correct protocol
$P$.

We first show that each bad computation $\xi$ is of length $> 4$.
There are 4 agents, so some agent, say $a$, is involved neither in the
first nor the second call of $\xi$.  So after these two calls none of
the agents $b, c, d$ is familiar with the secret of agent $a$. Each
call increases the number of agents familiar with the secret of agent
$a$ by at most 1. So at least three more calls are needed to obtain a
situation which all agents are familiar with the secret of $a$.  We
conclude that $\xi$ is of length at least 5.

The conclusion now follows by Lemma \ref{lem:bad}.
\end{proof}

Problem \ref{pro:bound} concerned the lower bound $2n-4$.
It is useful to note that the lower bound $2n-3$ can be achieved. So for $n=4$ the precise
bound is 5.

\begin{theorem} \label{thm:2n-3} Suppose that $n \geq 4$.  There
  exists a correct gossip protocol for $n$ agents all computations of which are of
  length $2n-3$.
\end{theorem}

\begin{proof}
Assume that $\Agents = \{a_1, \LL, a_n\}$.
We use the observation due to \cite{ADGH14a}
 that the following call sequence results in all agents 
 being experts:
 \[
  \begin{array}{l}
    a_1 a_2, a_1 a_3, \LL, a_1 a_n, \\
    a_1 a_2, a_1 a_3, \LL, a_1 a_{n-1}.
  \end{array}
\]
In this call sequence first agent $a_1$ calls all other agents and
subsequently calls them again except agent $a_{n-1}$. But the exact
order of the calls in each phase is not important. This allows us to
use a gossip protocol with a simple structure.

We choose an arbitrary agent $a \in \Agents$ and
select for it the following program:
\begin{align*}
*[&[]_{i \in \Agents \setminus \{a\}} \neg F_{a} I \to a i \\
    &[]_{i \in \Agents \setminus \{a\}} \Exp_{a} \land \neg K_{a} \Exp_{i} \to a i].
\end{align*}
The programs for the other agents are empty.

Note that after each call the corresponding guard becomes false.
Further, by the form of the above program in every computation of the
protocol first $n-1$ calls of $a$ take place, each to a different
agent. These calls correspond to the first set of guards and are
executed in an arbitrary order.

After the last of these $n-1$ calls, say $a b$, agents $a$
and $b$ both become experts and agent $a$ knows it. So from this
moment on the guard $\Exp_{a} \land \neg K_{a} \Exp_{b}$ is false
and the second call $ab$ does not take place.
Hence the remaining part of the computation consists of $n-2$ calls of
agent $a$, each to a different agent from $\Agents \setminus
\{a,b\}$. These calls correspond to the second set of guards and are
executed in an arbitrary order, as well.  

After these $2n-3$ calls all agents become experts and the protocol terminates.
\end{proof}

\begin{problem} \label{pro:bound2} Prove that the lower bound $2n-3$
  cannot be achieved by a correct propositional gossip protocol.
\end{problem}

\section{Open problems about verification of gossip protocols}
\label{sec:gp2}

We studied in \cite{AKW17} the computational complexity of partial
correctness and termination of gossip protocols and showed
that when in guards only formulas from the sublanguage $\lwn$ are used
both problems are in $\text{coNP}^{\text{NP}}$. This brings us naturally
to the following problems.

\begin{problem} \label{pro:pc}
  What is the exact computational complexity of checking whether a
  given gossip protocol
  \begin{itemize}
  \item is partially correct,
    
  \item terminates?
  \end{itemize}
\end{problem}

We conjecture that both problems are $\text{coNP}^{\text{NP}}$-complete.
This problem can be generalized as follows. In
{\em generalized gossiping} the aim is to reach a gossip situation in
which some given formula $\phi \in \lang$ is true, possibly other than
$\bigwedge_{a \in \Agents} Exp_a$. We say that a a gossip protocol $P$
is $\phi$-\bfe{partially correct} if for all leafs $\CSequence$ of the
computation tree of $P$
\[
  \CSequence \models \phi.
\]

For example the HMS gossip protocol from Example \ref{exa:hms} is
$\bigwedge_{i,j \in \Agents} K_i F_j I$-partially correct, since upon
termination all the guards are false. In contrast, this protocol is
not $\bigwedge_{i,j \in \Agents} K_i\, \Exp_j$-partially correct. Indeed,
assume three agents, $a,b,c$, and the call sequence $(ab, ac, bc)$
after which this gossip protocol terminates. However,
$(ab, ac, bc) \models K_a\, \Exp_b$ does not hold.

\begin{problem}
What is the complexity of checking whether for a given formula $\phi \in \lang$
a given gossip protocol is $\phi$-partially correct?
\end{problem}

These questions can also be asked for parametrized gossip protocols.

\begin{problem}
  What is the complexity of checking whether a given gossip protocol for
  arbitrary connected digraphs
  \begin{itemize}
  \item is partially correct,
   
  \item terminates?
  \end{itemize}
\end{problem}

Next problem concerns relation between two gossip protocols.
To define it we need to discuss the computation trees.

We say that two unordered trees are \bfe{equal} if they become identical after
some possible rearrangement of the children of each node in these
trees.\footnote{This rearrangement is allowed because computational
  trees are unordered.  If we fix an order of children, e.g., subject
  to lexicographic order on the last call, then a computational tree
  for a given protocol is unique and this problem when checking for the equality of two trees
  no longer occurs.}
Given a tree $T$, any removal of branches from $T$ yields a \bfe{subtree} of $T$.

Consider now two gossip protocols $P$ and $P'$.  We say that
$P$ \bfe{can simulate} $P'$ if the computational tree of $P'$ is equal
to some subtree of $P$.  If both $P$ can simulate $P'$ and $P'$ can
simulate $P$ then we say that they are \bfe{bisimilar}.  Clearly the
computational trees of two bisimilar gossip protocols are equal.

\begin{example} \label{exa:lns-hms}
  Consider the LNS and HMS gossip protocols introduced in Examples \ref{exa:lns} and
  \ref{exa:hms}.

Assume that there are only three agents, $a, b, c$.
There are only four maximal call sequences that can be generated by
LNS beginning with the call $ab$, namely:
\[
\mbox{$ab.bc.ac$, $ab.cb.ac$, $ab.ac.bc$
  $ab.ca.bc$.}
\]

For every other starting call $ac$, $ba$, $bc$, $ca$, or
$cb$, there are analogous four such maximal call sequences.  It is
easy to check that all these call sequences can be generated by HMS.
So for three agents HMS can simulate LNS.

However, HMS can also generate the call sequence $ab.bc.ca$ that LNS
cannot, so these two protocols are not bisimilar.
\HB
\end{example}

We actually conjecture that HMS can simulate LNS for any number of
agents. This naturally suggests the following two problems.

\begin{problem}
  What is the exact complexity of checking for two given parametric
  protocols whether
    \begin{itemize}
  \item one simulates the other,
    
  \item they are bisimilar?
  \end{itemize}
\end{problem}

\begin{problem}
  What is the exact computational complexity of checking for two
  gossip protocols $P$ and $P'$ whether
        \begin{itemize}
        \item $P$ simulates $P'$,
          
        \item $P$ and $P'$ are bisimilar?
\end{itemize}
\end{problem}

\NI
\textbf{Comments}

The following result provides some insights into the last problem.

\begin{theorem} \label{thm:simulate}
	Checking whether a gossip protocol $P$ can simulate another protocol $P'$ 
	can be done in exponential time.
\end{theorem}

To prove it we need some preparatory results.  We call the second call
$\Call$ in a call sequence
$\CSequence_1.\Call.\CSequence_2.\Call.\CSequence_3$
\bfe{epistemically redundant} if
$\CSequence_1.\Call(\init) =
\CSequence_1.\Call.\CSequence_2.\Call(\init)$.  In \cite{AKW17} we
established the following result (as Lemma 6) showing that removing an
epistemically redundant call does not affect the truth of any formula
from $\lwn$.

\begin{lemma}\label{lem:epistemic}
If $\CSequence_1.\Call.\CSequence_2.\Call.\CSequence_3$ is a call sequence 
where the second call $\Call$ is epistemically redundant, then
for any formula $\psi \in \lwn$: 
\[
  \CSequence_1.\Call.\CSequence_2.\Call.\CSequence_3 \models \psi \text{ iff }
  \CSequence_1.\Call.\CSequence_2.\CSequence_3 \models \psi.
\]
\end{lemma}

We say that a call $\Call$ that appears in a call sequence
$\CSequence$ is \bfe{productive} if
$\CSequence_1(\init) \neq \CSequence_1.\Call(\init)$, where for some
call sequences $\CSequence_1$ and $\CSequence_2$,
$\CSequence = \CSequence_1.\Call.\CSequence_2$.

Further, we call a subsequence $\CSequence_2$ of a call sequence
$\CSequence$ \bfe{stationary} if
\[
\CSequence_1.\Call_1(\init) = \CSequence_1.\Call_1.\Call_2(\init) = \dots = \CSequence_1.\CSequence_2(\init),
\]
where $\CSequence = \CSequence_1.\CSequence_2.\CSequence_3$ for some
call sequences $\CSequence_1$ and $\CSequence_3$, and for some
$k \geq 1$, $\CSequence_2 = \Call_1.\Call_2 \LL \Call_k$.

Recall that $n$ is the number of agents.

\begin{lemma}\label{lem:n2}
  Every call sequence contains at most $n^2-n$ productive calls.
\end{lemma}

\begin{proof}
  The minimal total number of secrets in a gossip situation is $n$,
  which is achieved in the initial gossip situation $\init$. In turn,
  the maximal total number of secrets is $n^2$, which is achieved in
  the gossip situation in which every agent is an expert.  After each
  productive call at least one agent learns a new secret.  So in a given
  call sequence there can be at most $n^2-n$ productive calls.
\end{proof}

\begin{lemma}\label{lem:n4}
  Every call sequence $\CSequence$ of length $\geq n^4$ has a stationary subsequence
   of length $\geq n^2$.
\end{lemma}
\begin{proof}
  Split $\CSequence$ into a sequence of $n^2$ subsequences, each of
  length at least $n^2$. Suppose none of these subsequences is
  stationary. Then each of them contains a productive call. So
  $\CSequence$ contains at least $n^2$ productive calls, which
  contradicts Lemma \ref{lem:n2}.
\end{proof}

\begin{corollary} \label{cor:redundant}
  Every call sequence $\CSequence$ of length $\geq n^4$ contains an epistemically
  redundant call.
\end{corollary}

\begin{proof}
  Take a stationary subsequence $\CSequence_1$ of $\CSequence$
  guaranteed by Lemma \ref{lem:n4}.  There are at most $n^2-n$
  different calls, so some call $\Call$ appears in $\CSequence_1$ at
  least twice. Its second occurrence in $\CSequence_1$ is then
  epistemically redundant in $\CSequence$.
\end{proof}

Next lemma shows that lack of simulation entails existence of a specific call sequence
of a bounded length.
\begin{lemma} \label{lem:short} If $P$ cannot simulate $P'$ then there
  exist a call sequence $\CSequence$ of length $\leq n^4$ that
  can be generated by $P'$, but not by $P$.
\end{lemma}
\begin{proof}
  If $P$ cannot simulate $P'$ then take a shortest rooted
  path in the computation tree of $P'$ that does not exist in the
  computation tree of $P$.  This path corresponds to some call
  sequence $\CSequence$.
		
  Suppose by contradiction that the length of $\CSequence$ exceeds $n^4$.
  By Corollary \ref{cor:redundant},
  $\CSequence$ can be partitioned into
  $\CSequence_1.\Call.\CSequence_2.\Call.\CSequence_3.\Call'$ such
  that
  $\CSequence_1.\Call(\init) =
  \CSequence_1.\Call.\CSequence_2.\Call(\init)$.  

  By Lemma \ref{lem:epistemic} for any formula $\phi \in \lwn$ and a subsequence $\CSequence_4$ of $\CSequence_3$
we have
  \begin{equation}
    \label{equ:iff}
    \mbox{$\CSequence_1.\Call.\CSequence_2.\Call.\CSequence_4 \models \phi$
      iff
  $\CSequence_1.\Call.\CSequence_2.\CSequence_4 \models \phi$.}
\end{equation}

By the definition of  $\CSequence$ the call sequence
$\CSequence_1.\Call.\CSequence_2.\Call.\CSequence_3$ can  be generated by $P$, while
the call sequence $\CSequence$ cannot.  So $P$
does not have a rule $\phi \to \Call'$ such that
$\CSequence_1.\Call.\CSequence_2.\Call.\CSequence_3 \models \phi$.

Further, by (\ref{equ:iff}) the call sequence
$\CSequence_1.\Call.\CSequence_2.\CSequence_3$ can be generated by
$P$, while, again by (\ref{equ:iff}), the call sequence
$\CSequence_1.\Call.\CSequence_2.\CSequence_3.\Call'$ (so $\CSequence$
with the second indicated occurence of $\Call$ omitted) cannot.

On the other hand, by assumption $\CSequence$ can be generated by
$P'$, so $P'$ has a rule $\phi' \to \Call'$ such that
$\CSequence_1.\Call.\CSequence_2.\Call.\CSequence_3 \models \phi'$.
Together with (\ref{equ:iff}) this implies that also
$\CSequence_1.\Call.\CSequence_2.\CSequence_3.\Call'$ can be generated
by $P'$. This yields a contradiction with the definition of
$\CSequence$.
\end{proof}
\II

\NI
\emph{Proof of Theorem \ref{thm:simulate}.}  To show that $P$ can
simulate $P'$ it suffices by Lemma \ref{lem:short} to check that each
call sequence of length $\leq n^4$ which can be generated by $P'$ can
also be generated by $P$.

We showed in \cite{AKW17} that checking for a given call sequence
$\CSequence$ and a formula $\phi \in \lwn$ whether
$\CSequence \models \phi$ can be done in exponential time (actually in
$\text{coNP}^\text{NP}$ time).  This implies that checking whether a
given call sequence $\CSequence$ of length $\leq n^4$ can be generated
by a gossip protocol can be done in exponential time.

Consequently, checking whether $P$ can simulate $P'$ can be done in
exponential time, as well, since there are exponentially many call
sequences of length $\leq n^4$.
\HB

\begin{corollary}
	Checking whether gossip protocols $P$ and $P'$ are bisimilar
	can be done in exponential time.
\end{corollary}

As a side remark of independent interest we conclude the paper with
the following consequence of Corollary \ref{cor:redundant}.

\begin{corollary} \label{cor:inf} If a gossip protocol terminates then
  all its computations are of length $< n^4$.
\end{corollary}

To prove we use the following result from \cite{AW18} that for
the case of formulas from $\lwn$ is actually a special case of Lemma \ref{lem:epistemic}.

\begin{lemma}[Stuttering] \label{lem:stuttering}
Suppose that $\CSequence := \CSequence_1. \Call. \CSequence_2$ and
$\CSequenced := \CSequence_1. \Call. \Call. \CSequence_2$.
Then for all formulas $\phi \in \lang$,
$\CSequence \models \phi$ iff $\CSequenced \models \phi$.
\end{lemma}

\NI
\emph{Proof of Corollary \ref{cor:inf}.}  Consider a finite
computation of length $\geq n^4$ and the corresponding call sequence
$\CSequence$. By Corollary \ref{cor:redundant}, $\CSequence$ begins
with a call sequence $\CSequence_1.\Call.\CSequence_2.\Call$, where
$\CSequence_1.\Call(\init) =
\CSequence_1.\Call.\CSequence_2.\Call(\init)$.

  Let $\psi \to \Call$ be the rule used in the considered gossip protocol to
  generate the second occurrence of the call $\Call$.  Then
  $\CSequence_1.\Call.\CSequence_2 \models \psi$ and by Lemma
  \ref{lem:epistemic} for $\CSequence_3 = \epsilon$ we have
  $\CSequence_1.\Call.\CSequence_2.\Call \models \psi$.

  Hence by the repeated use of the Stuttering Lemma \ref{lem:stuttering}
  for $\CSequence_2 = \epsilon$ and all $i \geq 1$,
  $\CSequence_1.\Call.\CSequence_2.\Call^{i} \models \psi$.
  Consequently, after the call sequence
  $\CSequence_1.\Call.\CSequence_2$ is generated, the rule
  $\psi \to \Call$ can be repeatedly applied. Hence
  $\CSequence_1.\Call.\CSequence_2.\Call^{\omega}$ is an infinite
  sequence of calls that corresponds to an infinite computation.
  \HB

\section*{Acknowledgements}

We thank reviewers for useful suggestions concerning the presentation.
The second author was
partially supported by EPSRC grants EP/M027287/1 and
EP/P020909/1. 

\bibliographystyle{eptcs}

\bibliography{new}%

\end{document}